\documentclass[12pt]{article}
\pdfoutput=1
\usepackage{fullpage}
\usepackage{hyperref}
\usepackage{amsthm, amsmath, amsfonts, amssymb}
\usepackage{bbm}
\usepackage{cases}
\usepackage{graphicx}
\usepackage{pagecolor}
\usepackage{caption}
\usepackage{subcaption}
\usepackage{mathtools}
\usepackage{verbatim}
\usepackage{natbib}
\newtheorem{prop}{Proposition}
\newtheorem{theorem}{Theorem}

\newcommand{\footremember}[2]{%
    \footnote{#2}
    \newcounter{#1}
    \setcounter{#1}{\value{footnote}}%
}

\graphicspath{{./figs/}}

\begin{document}
\title{Optimal approximation of continuous functions by very deep ReLU networks\footnote{Accepted for presentation at Conference on Learning Theory (COLT) 2018}}
%\coltauthor{\Name{Dmitry Yarotsky} \Email{d.yarotsky@skoltech.ru}\\
% \addr Skolkovo Institute of Science and Technology, Skolkovo Innovation Center, 3 Nobel st., Moscow  143026, Russia
% \and
% \addr Institute for Information Transmission Problems, Bolshoy Karetny per. 19, build.1, Moscow 127051, Russia
% }
 \author{Dmitry Yarotsky\footremember{skoltech}{Skolkovo Institute of Science and Technology, Skolkovo Innovation Center, 3 Nobel st., Moscow  143026
 Russia}\footremember{iitp}{Institute for Information Transmission Problems, Bolshoy Karetny per. 19, build.1, Moscow 127051, Russia}\\
 \texttt{d.yarotsky@skoltech.ru}
 }
 \date{}
\maketitle

\begin{abstract} We consider approximations of general continuous functions on finite-dimensional cubes by general deep ReLU neural networks and study the approximation rates with respect to the modulus of continuity of the function and the total number of weights $W$ in the network. We establish the complete phase diagram of feasible approximation rates and show that it includes two distinct phases. One phase corresponds to slower approximations that can be achieved with constant-depth networks and continuous weight assignments. The other phase provides faster approximations at the cost of depths necessarily growing as a power law $L\sim W^{\alpha}, 0<\alpha\le 1,$ and with necessarily discontinuous weight assignments. In particular, we prove that constant-width fully-connected networks of depth $L\sim W$ provide the fastest possible approximation rate $\|f-\widetilde f\|_\infty = O(\omega_f(O(W^{-2/\nu})))$ that cannot be achieved with less deep networks.  
\end{abstract}

\section{Introduction}
Expressiveness of deep neural networks with piecewise-linear (in particular, ReLU) activation functions has been a topic of much theoretical research in recent years. The topic has many aspects, with connections to combinatorics \citep{montufar2014number,telgarsky2016benefits}, topology \citep{bianchini2014complexity}, Vapnik-Chervonenkis dimension \citep{bartlett1998almost,sakurai1999tight} and fat-shattering dimension \citep{kearns1990efficient,anthony2009neural}, hierarchical decompositions of functions \citep{mhaskar2016learning}, information theory \citep{petersen2017optimal}, etc.

Here we adopt the perspective of classical approximation theory, in which the problem of expressiveness can be basically described as follows. Suppose that $f$ is a multivariate function, say on the cube $[0,1]^\nu$, and has some prescribed regularity properties; how efficiently can one approximate $f$ by deep neural networks?  The question has been studied in several recent publications. Depth-separation results for some explicit families of functions have been obtained in \cite{safran2016depth, telgarsky2016benefits}. General upper and lower bounds on approximation rates for functions characterized by their degree of smoothness have been obtained in \cite{liang2016why, yarotsky2017nn}. \cite{hanin2017approximating, lu2017expressive} establish the universal approximation property and convergence rates for deep and ``narrow'' (fixed-width) networks. \cite{petersen2017optimal} establish convergence rates for approximations of discontinuous functions. Generalization capabilities of deep ReLU networks trained on finite noisy samples are studied in \cite{schmidthieber2017nonparametric}.

In the present paper we consider and largely resolve the following  question: what is the optimal rate of approximation of general continuous functions by deep ReLU networks, in terms of the number $W$ of network weights and the modulus of continuity of the function? Specifically, for any $W$ we seek a network architecture with $W$ weights so that for any continuous $f:[0,1]^\nu\to \mathbb R$, as $W$ increases, we would achieve the best convergence in the uniform  norm $\|\cdot\|_\infty$ when using these architectures to approximate $f$.

In the slightly different but closely related context of approximation on balls in Sobolev spaces $\mathcal W^{d,\infty}([0,1]^\nu),$ this question of optimal convergence rate has been studied in \cite{yarotsky2017nn}. That paper described ReLU network architectures with $W$ weights ensuring approximation with error $O(W^{-d/\nu}\ln^{d/\nu} W)$ (Theorem 1). The construction was linear in the sense that the network weights depended on the approximated function linearly. Up to the logarithmic factor, this approximation rate matches the optimal rate over all parametric models under assumption of continuous parameter selection (\cite{devore1989optimal}). It was also shown in Theorem 2 of \cite{yarotsky2017nn} that one can slightly (by a logarithmic factor) improve over this conditionally optimal rate by adjusting network architectures to the approximated function.  

On the other hand, it was proved in Theorem 4 of \cite{yarotsky2017nn}  that ReLU networks generally cannot provide approximation with accuracy better than $O(W^{-2d/\nu})$ -- a bound with the power $\frac{2d}{\nu}$ twice as big as in the previously mentioned existence result. As was shown in the same theorem, this bound can be strengthened for shallow networks. However, without imposing depth constraints, there was a serious gap between the powers $\frac{2d}{\nu}$ and $\frac{d}{\nu}$ in the lower and upper accuracy bounds that was left open in that paper.

In the present paper we bridge this gap in the setting of continuous functions (which is slightly more general than the setting of the Sobolev space of Lipschitz functions, $\mathcal W^{1,\infty}([0,1]^\nu)$, i.e. the case $d=1$). Our key insight is the close connection between approximation theory and VC dimension bounds. The lower bound on the approximation accuracy in Theorem 4 of \cite{yarotsky2017nn} was derived using the upper VCdim bound $O(W^2)$ from \cite{goldberg1995bounding}. More accurate upper and lower bounds involving the network depth $L$ have been given in \cite{bartlett1998almost,sakurai1999tight}. The recent paper \cite{bartlett2017nearly} establishes nearly tight lower and upper VCdim bounds: $cWL\ln (W/L)\le\operatorname{VCdim}(W,L)\le CWL\ln W$, where $\operatorname{VCdim}(W,L)$ is the largest VC dimension
of a piecewise linear network with $W$ weights and $L$ layers. The key element in the proof of the lower bound is the ``bit extraction technique'' (\cite{bartlett1998almost}) providing a way to compress significant expressiveness in a single network weight.  In the present paper we adapt this technique to the approximation theory setting.

Our main result is the complete phase diagram for the parameterized family of approximation rates involving the modulus of continuity $\omega_f$ of the function $f$ and the number of weights $W$. We prove that using very deep networks one can approximate function $f$ with error $O(\omega_f(O(W^{-2/\nu})))$, and this  rate is optimal up to a logarithmic factor. In fact, the depth of the networks must necessarily grow almost linearly with $W$ to achieve this rate, in sharp contrast to  shallow networks that can provide approximation with error $O(\omega_f(O(W^{-1/\nu})))$. Moreover, whereas the slower rate $O(\omega_f(O(W^{-1/\nu})))$ can be achieved using a continuous weight assignment in the network, the optimal $O(\omega_f(O(W^{-2/\nu})))$ rate necessarily requires a discontinuous weight assignment. All this allows us to regard these two kinds of approximations as being in different ``phases''. In addition, we explore the intermediate rates $O(\omega_f(O(W^{-p})))$ with $p\in(\frac{1}{\nu},\frac{2}{\nu})$ and show that they are also in the discontinuous phase and require network depths $\sim W^{p\nu-1}.$ We show that the optimal rate $O(\omega_f(O(W^{-2/\nu})))$ can be achieved with a deep constant-width fully-connected architecture, whereas the rates $O(\omega_f(O(W^{-p})))$ with $p\in(\frac{1}{\nu},\frac{2}{\nu})$ and depth $O(W^{p\nu-1})$ can be achieved by stacking the deep constant-width architecture with a shallow parallel architecture.  Apart from the bit extraction technique, we use the idea of the two-scales expansion from Theorem 2 in \cite{yarotsky2017nn} as an essential tool in the proofs of our results.

We formulate precisely the results in Section  \ref{sec:result}, discuss them in Section \ref{sec:discus}, and give the proofs in Sections \ref{sec:proof1}, \ref{sec:proofmain}. 

\section{The results}\label{sec:result}
We define the modulus of continuity $\omega_f$ of a function $f:[0,1]^\nu\to\mathbb R$ by \begin{equation}\label{eq:omega}\omega_f(r)=\max\{|f(\mathbf x)-f(\mathbf y)|: \mathbf x, \mathbf y\in [0,1]^\nu, |\mathbf x-\mathbf y|\le r\},\end{equation}
where $|\mathbf x|$ is the euclidean norm of $\mathbf x$. 

%\pagecolor{yellow!30!orange}
\begin{figure}
\begin{center}
\includegraphics[clip,trim=12mm 8mm 12mm 8mm, scale=0.5]{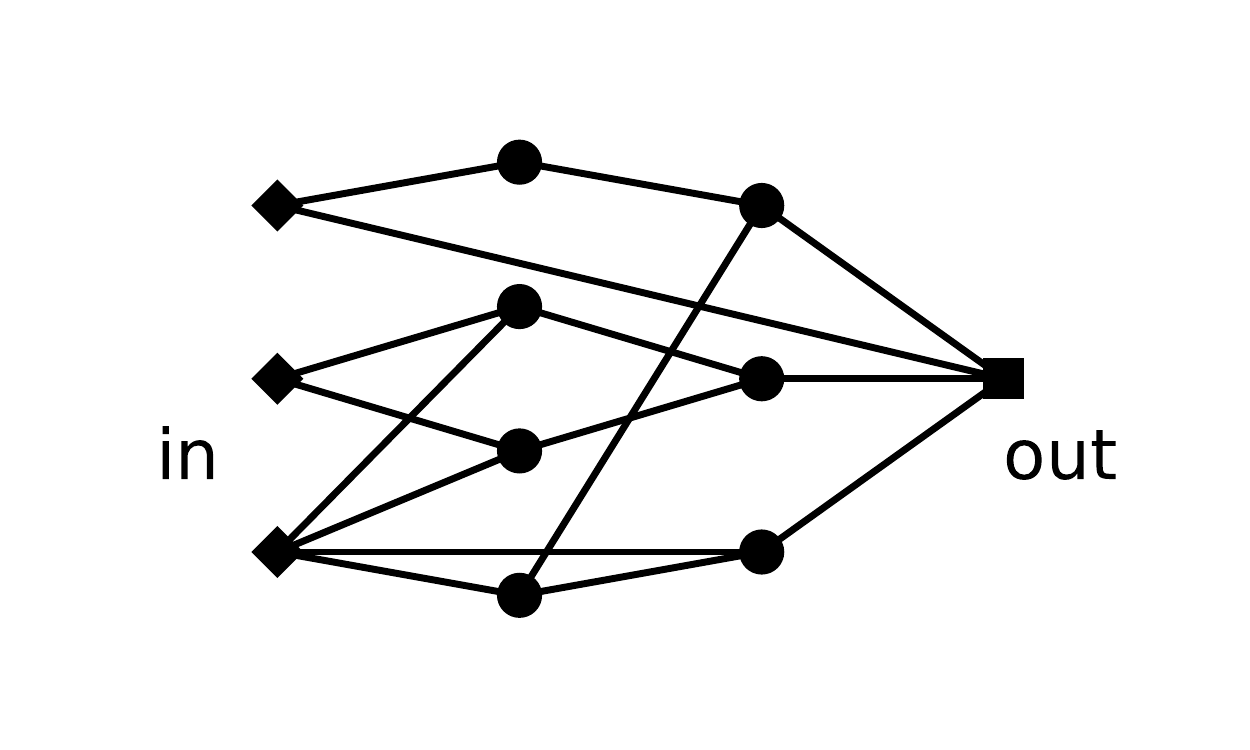}
\caption{An example of a feed-forward network architecture of depth $L=2$ with $W=24$ weights.}\label{fig:net}
\end{center}
\end{figure}

We approximate functions $f:[0,1]^\nu\to\mathbb R$ by usual feed-forward neural networks with the ReLU activation function $x\mapsto x_+\equiv \max(0,x)$.  The network has $\nu$ input units, some hidden units, and one output unit. The hidden units are assumed to be grouped in a sequence of layers so that the inputs of each unit is formed by outputs of some units from previous layers. The depth $L$ of the network is the number of these hidden layers. A hidden unit computes a linear combination of its inputs followed by the activation function: $x_1,\ldots,x_s\mapsto (\sum_{k=1}^s w_k x_k+h)_+$, where $w_k$ and $h$ are the weights associated with this unit. The output unit acts similarly, but without the activation function: $x_1,\ldots,x_s\mapsto \sum_{k=1}^s w_k x_k+h$. 

The network is determined by its architecture and weights. Clearly, the total number of weights, denoted by $W$, is equal to the total number of connections and computation units (not counting the input units). We don't impose any constraints on the network architecture (see Fig. \ref{fig:net} for an example of a valid architecture). 

Throughout the paper, we consider the input dimension $\nu$ as fixed. Accordingly, by \emph{constants} we will generally mean values that may depend on $\nu$.

We are interested in relating the approximation errors to the complexity of the function $f$, measured by its modulus of continuity $\omega_f$, and to the complexity of the approximating network, measured by its total number of weights $W$. More precisely, we consider  approximation rates in terms of the following procedure. 

First, suppose that for each $W$ we choose in some way a network architecture $\eta_W$ with $\nu$ inputs and $W$ weights. Then, for any $f:[0,1]^\nu\to\mathbb R$ we construct an approximation $\widetilde f_W:[0,1]^\nu\to \mathbb R$ to $f$ by choosing in some way the values of the weights in the architecture $\eta_W$  -- in the sequel, we refer to this stage as the \emph{weight assignment}. The question we ask is this: for which powers $p\in\mathbb R$ can we ensure, by the choice of the architecture and then the weights, that
\begin{equation}\label{eq:mainineq}
\|f-\widetilde f_W\|_\infty\le a\omega_f(c W^{-p}), \quad \forall f\in C([0,1]^\nu),
\end{equation}
with some constants $a, c$ possibly depending on $\nu$ and $p$ but not on $W$ or $f$?

Clearly, if inequality \eqref{eq:mainineq} holds for some $p$, then it also holds for any smaller $p$. However, we expect that for smaller $p$ the inequality can be in some sense easier to satisfy. In this paper we show that there is in fact a  \emph{qualitative} difference between different regions of $p$'s. 

%\pagecolor{yellow!30!orange}
\begin{figure}
\begin{center}
\includegraphics[clip,trim=20mm 30mm 30mm 35mm, scale=0.8]{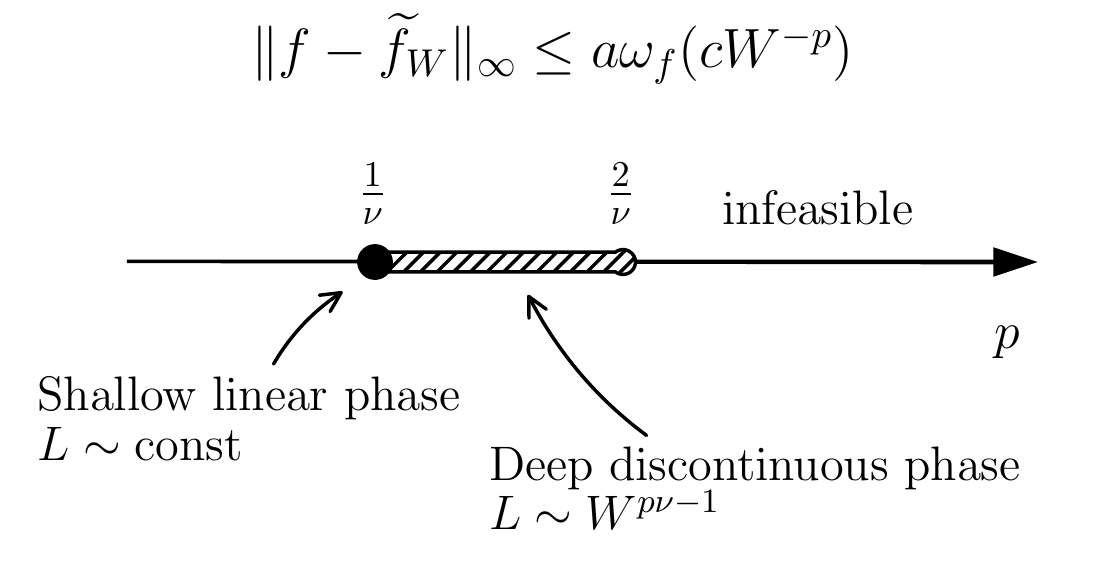}
\caption{The phase diagram of convergence rates. At $p=\frac{1}{\nu}$ the rate is achieved by shallow networks with weights linearly (and continuously) depending on $f$. At $p\in(\frac{1}{\nu},\frac{2}{\nu}]$, the rate is achieved by deep networks with weights discontinuously depending on $f$. Rates with $p>\frac{2}{\nu}$ are infeasible.}\label{fig:phasediag}
\end{center}
\end{figure}

Our findings are best summarized by the phase diagram shown in Fig. \ref{fig:phasediag}. We give an informal overview of the diagram before moving on to precise statements. The region of generally feasible rates is $p\le \frac{2}{\nu}.$ This region includes two qualitatively distinct phases corresponding to $p=\frac{1}{\nu}$ and $p\in (\frac{1}{\nu}, \frac{2}{\nu}]$. At $p=\frac{1}{\nu},$ the rate \eqref{eq:mainineq} can be achieved by fixed-depth networks whose weights depend linearly on the approximated function $f$. In contrast, at $p\in (\frac{1}{\nu}, \frac{2}{\nu}]$ the rate can only be achieved by networks with growing depths $L\sim W^{p\nu-1}$ and whose weights depend \emph{discontinuously} on the approximated function. In particular, at the rightmost feasible point $p=\frac{2}{\nu}$ the approximating architectures have $L\sim W$ and are thus necessarily extremely deep and narrow. 

%\pagecolor{yellow!30!orange}
\begin{figure}[t]
\begin{center}
%\begin{subfigure}[b]{0.3\textwidth}
\includegraphics[clip,trim=5mm 10mm 0mm 10mm, scale=0.8]{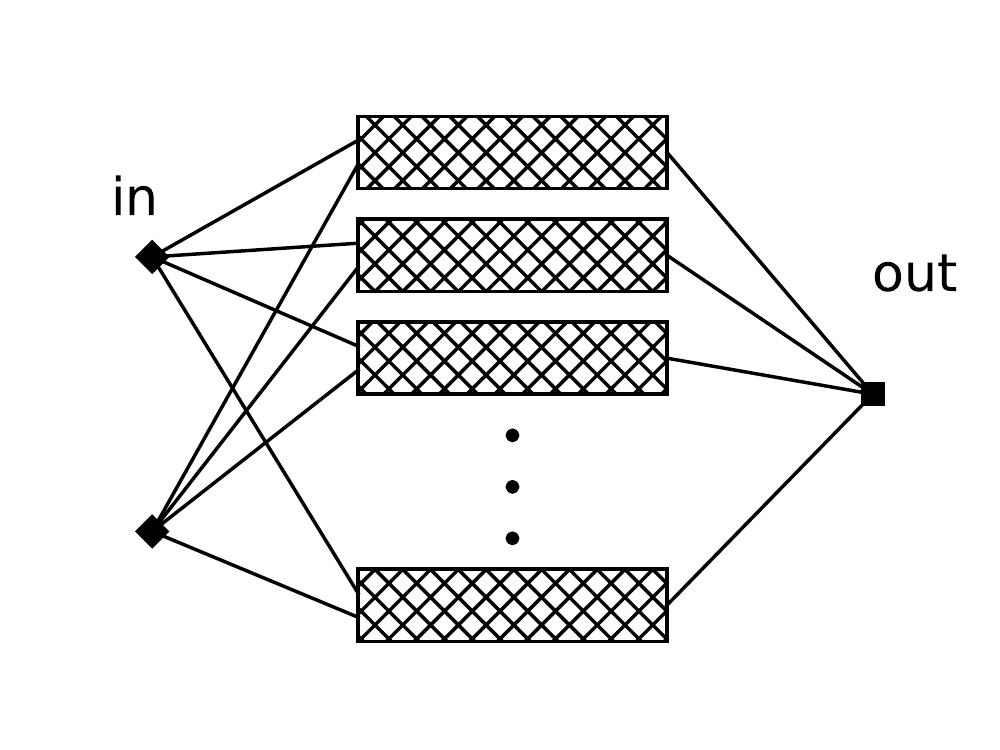}
\caption{The parallel, constant-depth network architecture implementing piecewise linear interpolation and ensuring approximation rate \eqref{eq:mainineq} with $p=\frac{1}{\nu}$.}\label{fig:shallow}
%\end{subfigure}
%\begin{subfigure}[b]{0.3\textwidth}
\includegraphics[clip,trim=120mm 30mm 120mm 105mm, scale=0.15]{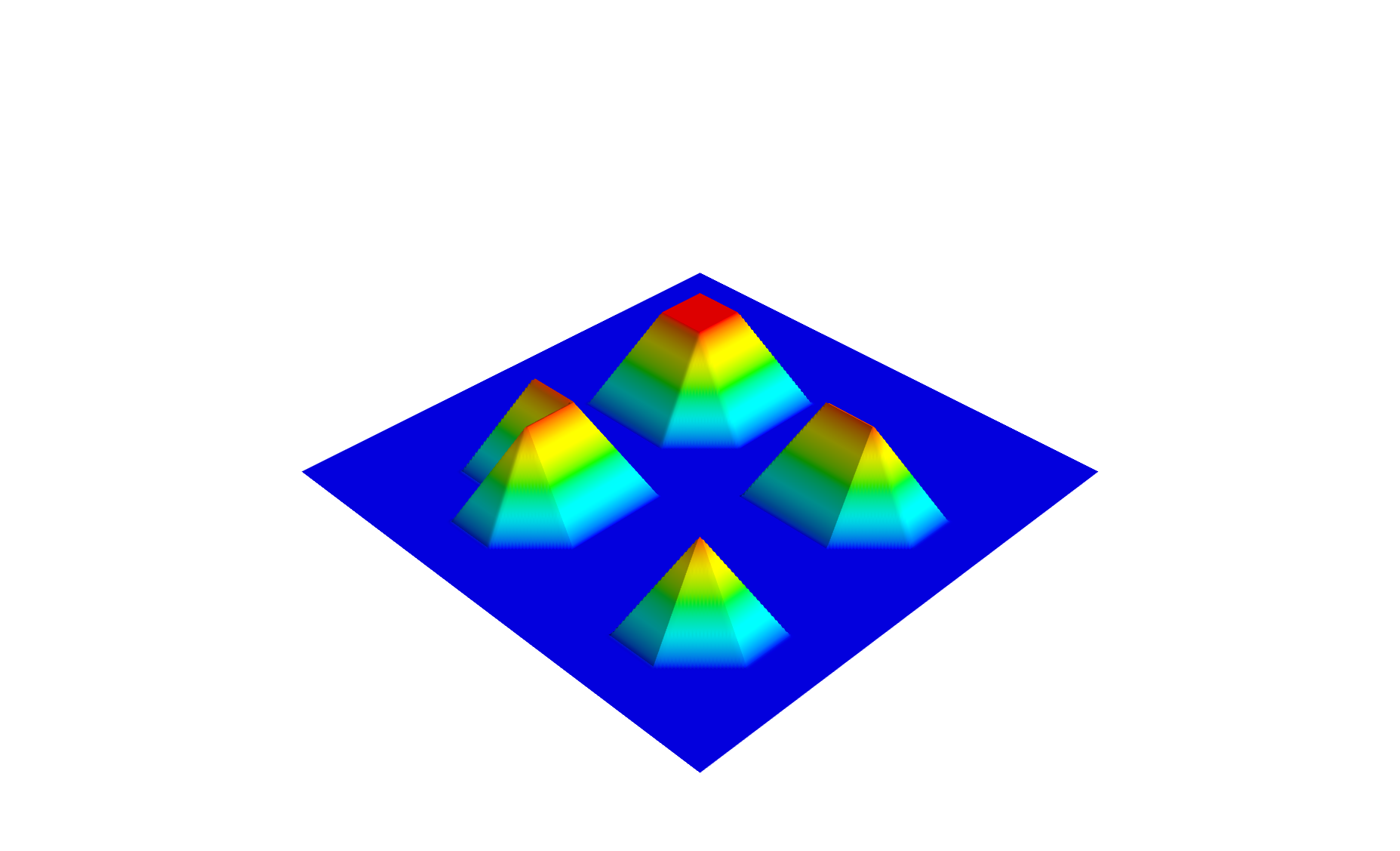}
\includegraphics[clip,trim=120mm 30mm 120mm 125mm, scale=0.15]{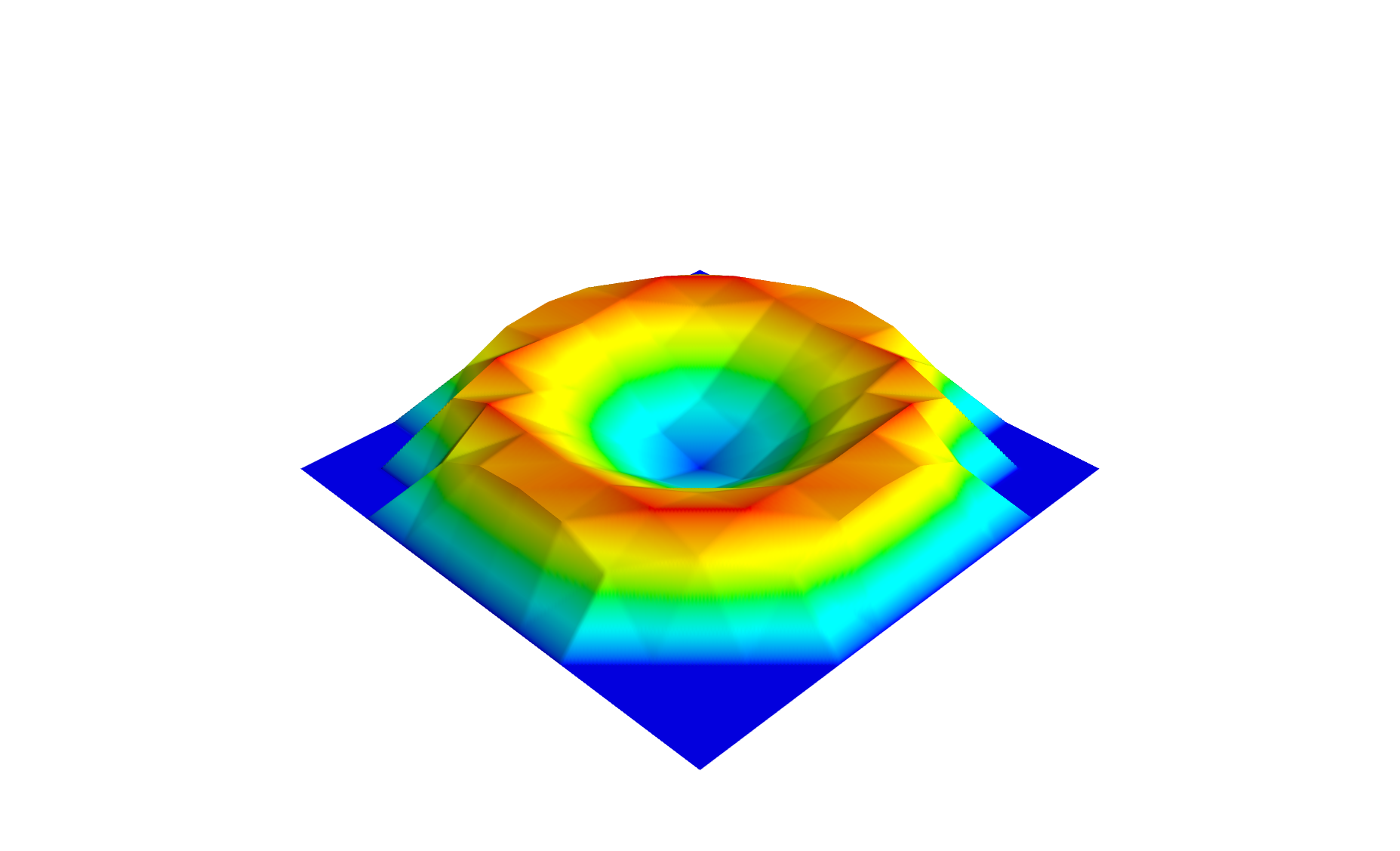}
\caption{Approximation by linear combinations of ``spike'' functions in dimension $\nu=2$. Left: a spike function and examples of sums of neighboring spike functions. Right: approximation of a radial function by a linear combination of spike functions.}\label{fig:spikes}
%\end{subfigure}
\end{center}
\end{figure}

We now turn to precise statements. First we characterize the $p=\frac{1}{\nu}$ phase in which the approximation can be obtained using a standard piecewise linear interpolation. In the sequel, when writing $f=O(g)$ we mean that $|f|\le c g$ with some constant $c$ that may depend on $\nu.$  For brevity, we will write $\widetilde f$ without the subscript $W$.  

\begin{prop}\label{prop:shallow} There exist network architectures $\eta_W$ with $W$ weights and, for each $W$, a weight assignment linear in $f$ such that Eq. \eqref{eq:mainineq} is satisfied with $p=\frac{1}{\nu}.$  
The network architectures can be chosen as consisting of $O(W)$ parallel blocks each having the same architecture that only depends on $\nu$ (see Fig. \ref{fig:shallow}). In particular, the depths of the networks depend on $\nu$ but not on $W$. 
\end{prop}
The detailed proof is given in Section  \ref{sec:proofshallow}; we explain now the idea. The approximating function $\widetilde f$ is constructed as a linear combination of ``spike'' functions sitting at the knots of the regular grid in $[0,1]^\nu$, with coefficients given by the values of $f$ at these knots (see Fig. \ref{fig:spikes}). For a grid of spacing $\frac{1}{N}$ with an appropriate $N$, the number of knots is $\sim N^\nu$ while the approximation error is $O( \omega_f(O(\frac{1}{N}))).$ We implement each spike by a block in the network, and implement the whole approximation by summing blocks connected in parallel and weighted. Then the whole network has $O(N^\nu)$ weights and, by expressing $N$ as $\sim W^{1/\nu}$, the approximation error is $O(\omega_f(O(W^{-1/\nu}))$, i.e. we obtain the rate \eqref{eq:mainineq} with $p=\frac{1}{\nu}$.

We note that the weights of the resulting network either do not depend on $f$ at all or are given by $w=f(\mathbf x)$ with some $\mathbf x\in [0,1]^\nu.$ In particular, the weight assignment is continuous in $f$ with respect to the standard topology of $C([0,1]^\nu)$.   

We turn now to the region $p>\frac{1}{\nu}.$ Several properties of this region are either direct consequences or slight modifications of existing results, and it is convenient to combine them in a single theorem.

\begin{theorem}\label{th:2}{}\hfill
\begin{enumerate}
\item[a)] (Feasibility) Approximation rate \eqref{eq:mainineq} cannot be achieved with $p>\frac{2}{\nu}$.
\item[b)] (Inherent discontinuity) Approximation rate \eqref{eq:mainineq} cannot be achieved with $p>\frac{1}{\nu}$ if the weights of $\widetilde f$ are required to depend on $f$ continuously with respect to the standard topology of $C([0,1]^\nu)$.
\item[c)] (Inherent depth) If approximation rate \eqref{eq:mainineq} is achieved with a particular $p\in (\frac{1}{\nu},\frac{2}{\nu}]$, then the architectures $\eta_W$ must have depths $L\ge d W^{p\nu-1}/\ln W$  with some possibly $\nu$- and $p$-dependent constant $d>0.$
\end{enumerate}
\end{theorem}
\begin{proof} 
The proofs of these statements have the common element of considering the approximation for functions from the unit ball $F_{\nu,1}$ in the  Sobolev space  $\mathcal W^{1,\infty}([0,1]^\nu)$  of Lipschitz functions. Namely, suppose that the approximation rate \eqref{eq:mainineq} holds with some $p$. Then all $f\in F_{\nu,1}$ can be approximated by architectures $\eta_W$ with accuracy \begin{equation}\label{eq:epsw}\epsilon_W= c_1 W^{-p}\end{equation} with some constant $c_1$ independent of $W$. The three statements of the theorem are then obtained as follows.

\medskip\noindent
a) This statement is a consequence of Theorem 4a) of \cite{yarotsky2017nn}, which is in turn a consequence of the upper bound $O(W^2)$ for the VC dimension of a ReLU network (\cite{goldberg1995bounding}). Precisely, Theorem 4a) implies that if an architecture $\eta_W$ allows to approximate all $f\in F_{\nu,1}$ with accuracy $\epsilon_W$, then $W\ge c_2\epsilon_W^{-\nu/2}$ with some $c_2$. Comparing this with Eq. \eqref{eq:epsw}, we get $p\le \frac{2}{\nu}.$

\medskip\noindent 
b) This statement is a consequence of the general bound of \cite{devore1989optimal} on the efficiency of approximation of Sobolev balls with parametric models having parameters continuously depending on the approximated function. Namely, if the weights of the networks $\eta_W$ depend on $f\in F_{\nu,1}$ continuously, then Theorem 4.2 of \cite{devore1989optimal} implies that $\epsilon_W\ge c_2 W^{-1/\nu} $ with some constant $c_2$, which implies that $p\le\frac{1}{\nu}.$

\medskip\noindent 
c) This statement can be obtained by combining arguments of Theorem 4 of \cite{yarotsky2017nn} with the recently established tight upper bound for the VC dimension of ReLU networks (\cite{bartlett2017nearly}, Theorem 6) with given depth $L$ and the number of weights $W$: \begin{equation}\label{eq:vcub}\operatorname{VCdim}(W,L)\le CWL\ln W,\end{equation} where $C$ is a global constant. 

Specifically, suppose that an architecture $\eta_W$ allows to approximate all $f\in F_{\nu,1}$ with accuracy $\epsilon_W$. Then, by considering suitable trial functions, one shows that if we threshold the network output, the resulting networks must have VC dimension $\mathrm{VCdim}(\eta_W)\ge c_2\epsilon_W^{-\nu}$ (see Eq.(38) in \cite{yarotsky2017nn}). Hence, by Eq. \eqref{eq:epsw}, $\mathrm{VCdim}(\eta_W)\ge c_3 W^{p\nu}$. On the other hand,  the upper bound \eqref{eq:vcub} implies $\mathrm{VCdim}(\eta_W)\le CWL\ln W$. We conclude that $c_3W^{p\nu}\le CWL\ln W$, i.e. $L\ge dW^{p\nu-1}/\ln W$ with some constant $d$. 
\end{proof}

Theorem \ref{th:2} suggests the existence of an approximation phase drastically different from the phase $p=\frac{1}{\nu}$. This new phase would provide better approximation rates, up to $p=\frac{2}{\nu}$, at the cost of deeper networks and some complex, discontinuous weight assignment. The main contribution of the present paper is the proof that this phase indeed exists. 

We describe some architectures that, as we will show, correspond to this phase. First we describe the architecture for $p=\frac{2}{\nu},$ i.e. for the fastest possible approximation rate. Consider the usual fully-connected architecture connecting neighboring layers and having a constant number of neurons in each layer, see Fig. \ref{fig:stnet}. We refer to this constant number of neurons as the  ``width'' $H$ of the network. Such a network of width $H$ and depth $L$ has $W=L(H^2+H)+H^2+(\nu+1)H+1$ weights in total. We will be interested in the scenario of ``narrow'' networks where $H$ is fixed and the network grows by increasing $L$; then $W$ grows linearly with $L$. Below we will refer to the ``narrow fully-connected architecture of width $H$ having $W$ weights'': the depth $L$ is supposed in this case to be determined from the above equality; we will assume without loss of generality that  the equality is satisfied with an integer $L$. We will show that these narrow architectures provide the $p=\frac{2}{\nu}$ approximation rate if the width $H$ is large enough (say, $H=2\nu+10$).   

In the case $p\in (\frac{1}{\nu},\frac{2}{\nu})$ we consider another kind of architectures obtained by stacking parallel shallow architectures (akin to those of Proposition \ref{prop:shallow}) with the above narrow fully-connected architectures, see Fig. \ref{fig:stackednet}. The first, parallelized part of these architectures consists of blocks that only depend on $\nu$ (but not on $W$ or $p$). The second, narrow fully-connected part will again have a fixed width,  and we will take its depth to be $\sim W^{p\nu-1}$. All the remaining weights then go into the first parallel subnetwork, which in particular determines the number of blocks in it. Since the blocks are parallel and their architectures do not depend on $W$, the overall depth of the network is determined by the second, deep subnetwork and is $O(W^{p\nu-1})$. On the other hand, in terms of the number of weights, for $p<\frac{2}{\nu}$ most computation is performed by the first, parallel subnetwork (the deep subnetwork has $O(W^{p\nu-1})$ weights while the parallel one has an asymptotically larger number of weights, $W-O(W^{p\nu-1})$). 

Clearly, these stacked architectures can be said to ``interpolate'' between the purely parallel architectures for $p=\frac{1}{\nu}$ and the purely serial architectures for $p=\frac{2}{\nu}$. Note that a parallel computation can be converted into a serial one at the cost of increasing the depth of the network. For $p<\frac{2}{\nu}$, rearranging the parallel subnetwork of the stacked architecture into a serial one would destroy the $O(W^{p\nu-1})$ bound on the depth of the full network, since the parallel subnetwork has $\sim W$ weights. However, for $p=\frac{2}{\nu}$ this rearrangement does not affect the $L\sim  W$ asymptotic of the depth more than by a constant factor  -- that's why we don't include the parallel subnetwork into the full network in this case.  

%\pagecolor{yellow!30!orange}
\begin{figure}[t]
\begin{center}
\includegraphics[clip,trim=15mm 5mm 5mm 5mm, scale=0.8]{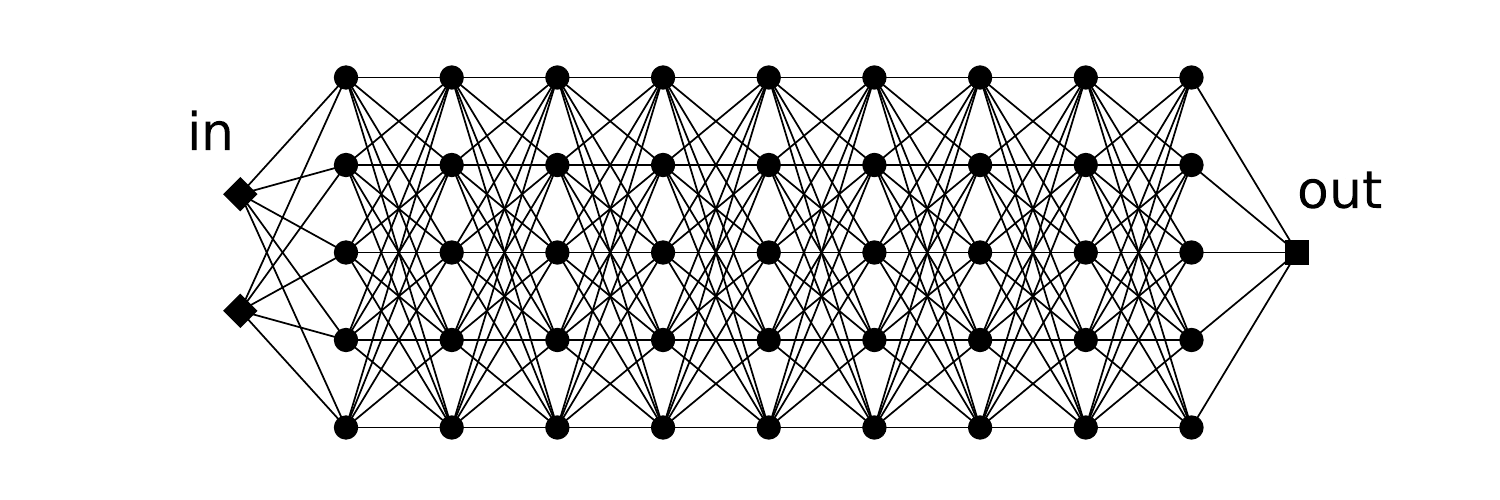}
\caption{An example of ``narrow'' fully-connected network architecture having $\nu=2$ inputs,   depth $L=9$ and width $H=5$. These architectures provide the optimal approximation rate \eqref{eq:mainineq} with $p=\frac{2}{\nu}$ if $H$ is sufficienly large and held constant while $L$ is increased.}\label{fig:stnet}
\includegraphics[clip,trim=15mm 5mm 5mm 5mm, scale=0.8]{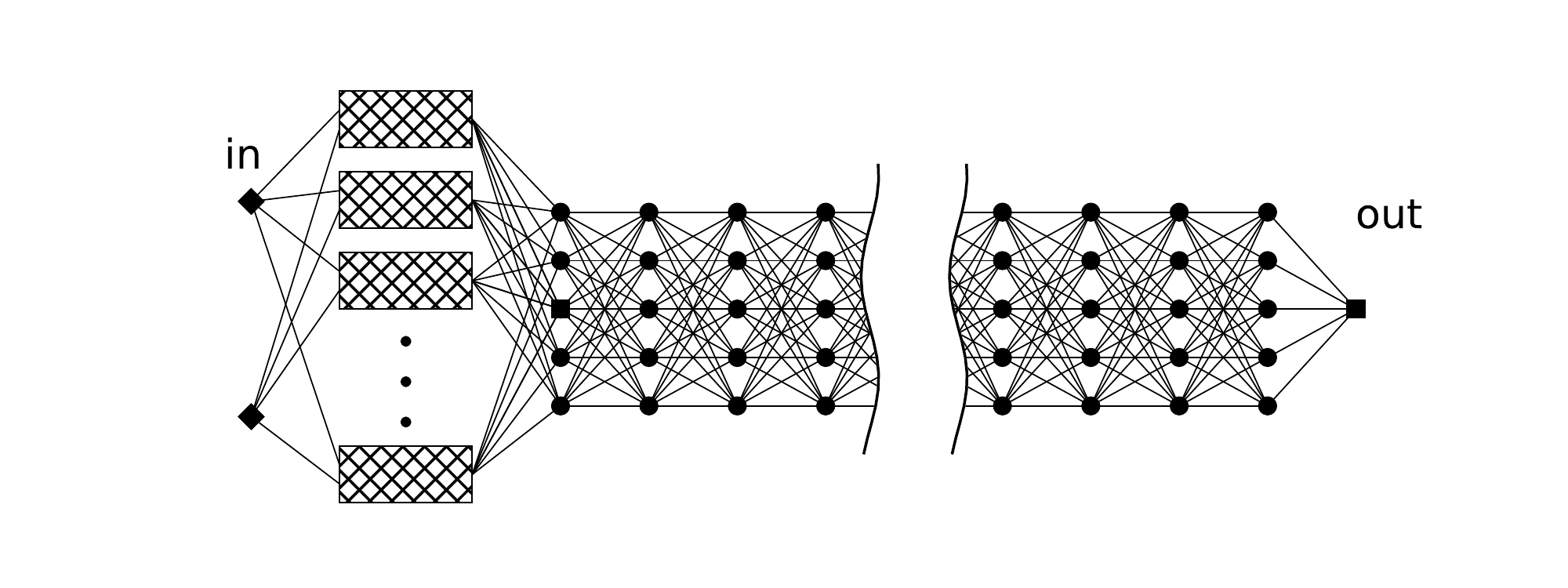}
\caption{The ``stacked'' architectures for $p\in(\frac{1}{\nu},\frac{2}{\nu})$, providing the optimal approximation rates \eqref{eq:mainineq} under the depth constraint $L=O(W^{p\nu-1})$.}\label{fig:stackednet}
\end{center}
\end{figure}

We state now our main result as the following theorem.

\begin{theorem}\label{th:main}{}\hfill
\begin{enumerate}
\item[a)] 
For any $p\in (\frac{1}{\nu},\frac{2}{\nu}]$, there exist a sequence of architectures $\eta_W$ with depths $L=O(W^{p\nu-1})$ and respective weight assignments such that inequality \eqref{eq:mainineq} holds with this $p$. 
\item[b)] For $p=\frac{2}{\nu}$, an example of such architectures is the  narrow fully-connected architectures of constant width $2\nu+10$.
\item[c)] For $p\in(\frac{1}{\nu},\frac{2}{\nu})$, an example of such architectures are stacked architectures described above, with the narrow fully-connected subnetwork having width $3^\nu(2\nu+10)$ and depth $W^{p\nu-1}$.
\end{enumerate}
\end{theorem}
Comparing this theorem with Theorem \ref{th:2}a) we see that the narrow fully-connected architectures provide the best possible approximation in the sense of Eq. \eqref{eq:mainineq}. Moreover, for $p\in(\frac{1}{\nu}, \frac{2}{\nu})$ the upper bound on the network depth in Theorem \ref{th:main}c) matches the lower bound in Theorem \ref{th:2}c) up to a logarithmic factor. This proves that for $p\in(\frac{1}{\nu}, \frac{2}{\nu})$ our stacked architectures are also optimal (up to a logarithmic correction) if we additionally impose the asymptotic constraint $L=O(W^{p\nu-1})$ on the network depth.      

The full proof of Theorem \ref{th:main} is given in Section \ref{sec:proofmain}; we explain now its main idea. Given a function $f$ and some $W$, we first proceed as in Proposition \ref{prop:shallow} and construct its piecewise linear interpolation $\widetilde f_1$ on the length scale $\frac{1}{N}$ with $N\sim W^{1/\nu}$. This approximation has uniform error $O(\omega_f(O(W^{-1/\nu})))$. Then, we improve this approximation by constructing an additional approximation $\widetilde f_2$ for the discrepancy $f-\widetilde f_1$. This second approximation lives on a smaller length scale $\frac{1}{M}$ with $M\sim W^{p}$. In contrast to $\widetilde f_1$, the second approximation is inherently discrete: we consider a finite set of possible shapes of $f-\widetilde f_1$ in patches of linear size $\sim\frac{1}{N}$, and in each patch we use a special single network weight to encode the  shape closest to $f-\widetilde f_1$. The second approximation is then fully determined by the collection of these special encoding weights found for all patches. We make the  parallel subnetwork of the full network serve two purposes: in addition to computing the initial approximation $\widetilde f_1(\mathbf x)$ as in Proposition \ref{prop:shallow}, the subnetwork returns the position of $\mathbf x$ within its patch along with the weight that encodes the second approximation $\widetilde f_2$ within this patch. The remaining, deep narrow part of the network then serves to decode the second approximation within this patch from the special weight and compute the value $\widetilde f_2(\mathbf x)$. Since the second approximation lives on the smaller length scale $\frac{1}{M}$, there are $Z=\exp(O((M/N)^\nu))$ possible  approximations $\widetilde f_2$ within the patch that might need to be encoded in the special weight. It then takes a narrow network of depth $L\sim \ln Z$ to reconstruct the approximation from the special weight using the bit extraction technique of \cite{bartlett1998almost}. As $M\sim W^{p}$, we get $L\sim W^{p\nu-1}$. At the same time, the second approximation allows us to effectively improve the overall approximation scale from $\sim \frac{1}{N}$ down to $\sim \frac{1}{M}$, i.e. to $\sim W^{-p}$, while keeping the total number of weights in the network. This gives us the desired error bound $O(\omega_f(O(W^{-p}))).$ 

We remark that the discontinuity of the weight assignment in our construction is the result of the discreteness of the second approximation $\widetilde f_2$: whereas the variable weights in the network implementing the first approximation $\widetilde f_1$  are found by linearly projecting the approximated function to $\mathbb R$ (namely, by computing $f\mapsto f(\mathbf x)$ at the knots $\mathbf x$), the variable weights for $\widetilde f_2$ are found by assigning to $f$ one of the finitely many values encoding the possible approximate shapes in a patch. This operation is obviously discontinuous.  While the discontinuity is present for all $p>\frac{1}{\nu},$ at smaller $p$ it is ``milder'' in the sense of a smaller number of assignable values.

\section{Discussion}\label{sec:discus}

We discuss now our result in the context of general approximation theory and practical machine learning. First, a theorem of \cite{kainen1999approximation} shows that in the optimal approximations by neural networks the weights generally discontinuously depend on the approximated function, so the discontinuity property that we have established is not surprizing. However, this theorem of \cite{kainen1999approximation} does not in any way quantify the accuracy gain that can be acquired by giving up the continuity of the weights. Our result does this in the practically important case of deep ReLU networks, and explicitly describes a relevant mechanism.

In general, many nonlinear approximation schemes involve some form of discontinuity, often explicit (e.g., using different expansion bases for different approximated functions (\cite{devore1998nonlinear}). At the same time, discontinuous selection of parameters in parametric models is often perceived as an undesirable  phenomenon associated with unreliable approximation (\cite{devore1989optimal,devore1998nonlinear}). We point out, however, that deep  architectures considered in the present paper resemble some popular state-of-the-art practical networks for highly accurate image recognition -- residual networks \citep{he2016deep} and highway networks \citep{srivastava2015highway} that may have dozens or even hundreds of layers. While our model does not explicitly include shortcut connections as in ResNets, a very similar element is effectively present in the proof of Theorem \ref{th:main} (in the form of channels reserved for passing forward the data). We expect, therefore, that our result may help better understand the properties of ResNet-like networks. 

Quantized network weights have been previously considered from the information-theoretic point of view in \cite{bolcskei2017memory,petersen2017optimal}. In the present paper we do not use quantized weights in the statement of the approximation problem, but they appear in the solution (namely, we use them to store small-scale descriptions of the approximated function). One can expect that weight quantization may play an important role in the future development of the theory of deep networks.

\section{Preliminaries and proof of Proposition \ref{prop:shallow}}\label{sec:proof1}

\subsection{Preliminaries} 
The modulus of continuity defined by \eqref{eq:omega} is monotone nondecreasing in $r$. By the convexity of the cube $[0,1]^\nu,$ for any integer $n$ we have $\omega_f(r)\le n\omega_f(\frac{r}{n}).$ More generally, for any $\alpha\in (0,1)$ we can write \begin{equation}\label{eq:alphaomega}\alpha\omega_f(r)\le 1/\lfloor 1/\alpha\rfloor\omega_f(r)\le \omega_f(r/\lfloor 1/\alpha\rfloor)\le \omega_f(2\alpha r).\end{equation}

The ReLU function $x\mapsto x_+$ allows us to implement the binary $\max$ operation as $\max(a,b)=a+(b-a)_+$ and the binary min operation as $\min(a,b)=a-(a-b)_+$. The maximum or minimum of any $n$ numbers can then be implemented by chaining $n-1$ binary $\max$'s or $\min$'s.  Computation of the absolute value can be implemented by $|x|=2x_+-x.$

Without loss of generality, we can assume that hidden units in the ReLU network may not include the ReLU nonlinearity (i.e., may compute just a linear combination of the input values). Indeed, we can simulate nonlinearity-free units just by increasing the weight $h$ in the formula $x_1,\ldots,x_s\mapsto (\sum_{k=1}^s w_k x_k+h)_+$, so that $\sum_{k=1}^s w_k x_k+h$ is always nonnegative (this is possible since the network inputs are from a compact set and the network implements a continuous function), and then compensating the shift by subtracting appropriate amounts in the units receiving input from the unit in question. In particular, we can simulate in this way trivial ``pass forward'' (identity) units that simply deliver given values to some later layers of the network. 

In the context of deep networks of width $H$ considered in Section \ref{sec:proofmain}, it will be occasionally convenient to think of the network as consisting of $H$ ``channels'', i.e. sequences of units with one unit per layer.  Channels can be used to pass forward values or to perform computations. For example, suppose that we already have $n$ channels that pass forward $n$ numbers. If we also need to compute the maximum of these numbers, then, by chaining binary $\max$'s, this can be done in a subnetwork including one additional channel that spans $n$ layers. 

We denote vectors by boldface characters; the scalar components of a vector $\mathbf x$ are denoted $x_1, x_2,\ldots$

\subsection{Proof of Proposition \ref{prop:shallow}}\label{sec:proofshallow}
We start by describing the piecewise linear approximation of the function $f$ on a scale $\frac{1}{N},$ where $N$ is some fixed large integer that we will later relate to the network size $W$. It will be convenient to denote this approximation by  $\widetilde f_1$. This approximation is constructed as an interpolation of $f$ on the grid $(\mathbb Z/N)^\nu$. To this end, we consider the standard triangulation $P_N$ of the space $\mathbb R^\nu$ into the simplexes
$$\Delta^{(N)}_{\mathbf n, \rho}=\{\mathbf x\in\mathbb R^\nu: 0\le x_{\rho(1)}-\tfrac{n_{\rho(1)}}{N}\le\ldots\le x_{\rho(\nu)}-\tfrac{n_{\rho(\nu)}}{N}\le \tfrac{1}{N}\},$$ where $\mathbf n=(n_1,\ldots,n_\nu)\in\mathbb Z^\nu$ and $\rho$ is a permutation of $\nu$ elements. This triangulation can be described as resulting by dissecting the space $\mathbb R^\nu$ by the hyperplanes $x_k-x_s=\tfrac{n}{N}$ and $x_k=\tfrac{n}{N}$ with various $1\le k,s\le \nu$ and $n\in \mathbb Z.$

The vertices of the simplexes of the triangulation $P_N$ are the points of the grid $(\mathbb Z/N)^\nu$.  Each such point $\frac{\mathbf n}{N}$ is a vertex in $(\nu+1)!$ simplexes. The union of these $(\nu+1)!$ simplexes is the convex set $\{\mathbf x\in \mathbb R^\nu: (\max(\mathbf x-\tfrac{\mathbf n}{N}))_+-(\min(\mathbf x-\tfrac{\mathbf n}{N}))_-\le\tfrac{1}{N}\}$, where $a_-=\min(a,0).$  

We define the ``spike function'' $\phi:\mathbb R^\nu\to \mathbb R$ as  the continuous piecewise linear function  such that 1) it is linear on each simplex of the partition $P_{N=1}$, and 2) $\phi(\mathbf n)=\mathbbm 1[\mathbf n=\mathbf 0]$ for $\mathbf n\in \mathbb Z^\nu$. The spike function can be expressed in terms of linear and ReLU operations as follows. Let $R$ be the set of the $(\nu+1)!$ simplexes having $\mathbf 0$ as a vertex. For each $\Delta\in R$ let $l_\Delta:\mathbb R^\nu\to \mathbb R$ be the affine map such that $l_\Delta(\mathbf 0)=1$ and $l_{\Delta}$ vanishes on the face of $\Delta$ opposite to the vertex $\mathbf 0.$ Then \begin{equation}\label{eq:spike}\phi(\mathbf x)=\big(\min_{\Delta\in R}(l_\Delta(\mathbf x))\big)_+.\end{equation} 
Indeed, if $\mathbf x\notin \cup_{\Delta\in R}\Delta,$ then it is easy to see that we can find some $\Delta\in R$ such that $l_{\Delta}(\mathbf x)<0$, hence $\phi$ vanishes outside $\cup_{\Delta\in R}\Delta$, as required. On the other hand, consider the restriction of $\phi$ to $\cup_{\Delta\in R}\Delta$. Each $l_\Delta$ is nonnegative on this set. For each $\Delta_1$ and $\mathbf x\in \cup_{\Delta\in R}\Delta$, the value $l_{\Delta_1}(\mathbf x)$ is a convex combination of the values of $l_{\Delta_1}$ at the vertices of the  simplex $\Delta_{\mathbf x}$ that $\mathbf x$ belongs to. Since $l_{\Delta}$ is nonnegative and $l_{\Delta}(\mathbf 0)=1$ for all $\Delta$, the minimum in \eqref{eq:spike} is attained at $\Delta$ such that $l_{\Delta}$ vanishes at all vertices of $\Delta_{\mathbf x}$ other than $\mathbf 0$, i.e. at $\Delta=\Delta_{\mathbf x}$.  

Note that each map $l_\Delta$ in \eqref{eq:spike} is either of the form $1+x_k-x_s$ or $1\pm x_k$ (different simplexes may share the same map $l_\Delta$), so the minimum actually needs to be taken only over  $\nu(\nu+1)$ different $l_\Delta(\mathbf x)$: $\phi(\mathbf x)=\big(\min(\min_{k\ne s}(1+x_k-x_s), \min_k(1+x_k), \min_k(1-x_k))\big)_+.$ 

We define now the piecewise linear interpolation $\widetilde f_1$ by 
\begin{equation}\label{eq:wtf1}
\widetilde f_1(\mathbf x) = \sum_{\mathbf n\in\{0,1,\ldots,N\}^\nu}f(\tfrac{\mathbf n}{N})\phi(N\mathbf x-\mathbf n).
\end{equation} 
The function $\widetilde f_1$ is linear on each simplex $\Delta_{\mathbf n,\rho}^{(N)}$ and agrees with $f$ at the interpolation knots $\mathbf x\in (\mathbb Z/N)^\nu\cap [0,1]^\nu$. We can bound $\||\nabla \widetilde f_1|\|_\infty\le\sqrt{\nu}N\omega_f(\tfrac{1}{N})$, since any partial derivative $\partial_k \widetilde f_1$ in the interior of a simplex equals $N(f(\tfrac{\mathbf n_1}{N})-f(\tfrac{\mathbf n_2}{N}))$, where $\tfrac{\mathbf n_1}{N}, \tfrac{\mathbf n_2}{N}$ are the two vertices of the simplex having all coordinates the same except $x_k$. It follows that the modulus of continuity for $\widetilde f_1$ can be bounded by $\omega_{\widetilde f_1}(r)\le\sqrt{\nu}N\omega_f(\tfrac{1}{N})r$. Moreover, by Eq.\eqref{eq:alphaomega} we have $\frac{Nr}{2}\omega_f(\frac{1}{N})\le \omega_f(r)$ as long as $r<\frac{2}{N}.$ Therefore we can also write $\omega_{\widetilde f_1}(r)\le 2\sqrt{\nu}\omega_f(r)$  for $r\in[0,\frac{2}{N}].$

Consider now the discrepancy \[f_2=f-\widetilde f_1.\] We can bound the modulus of continuity  of $f_2$ by $\omega_{f_2}(r)\le \omega_{f}(r)+\omega_{\widetilde f_1}(r)\le (2\sqrt{\nu}+1)\omega_f(r)$ for $r\in[0,\frac{2}{N}].$ Since any point $\mathbf x\in[0,1]^\nu$ is within the distance $\tfrac{\sqrt{\nu}}{2N}$ of one of the interpolation knots $\tfrac{\mathbf n}{N}$ where $f_2$ vanishes, we can also write 
\begin{equation}\label{eq:f2infty}\|f_2\|_\infty\le 
\omega_{f_2}(\tfrac{\sqrt{\nu}}{2N})\le  
\sqrt{\nu} \omega_{f_2}(\tfrac{1}{N})\le  
\sqrt{\nu}(2\sqrt{\nu}+1)\omega_{f}(\tfrac{1}{N}) \le 
3\nu\omega_{f}(\tfrac{1}{N}),\end{equation} where in the second inequality we again used Eq.\eqref{eq:alphaomega}. 

We observe now that formula \eqref{eq:wtf1} can be represented by a parallel network shown in Fig. \ref{fig:shallow} in which the blocks compute the values $\phi(N\mathbf x-\mathbf n)$ and their output connections carry the weights $f(\tfrac{\mathbf n}{N})$. Since the blocks have the same architecture only depending on $\nu$, for a network with $n$ blocks the total number of weights is $O(n)$. It follows that the number of weights will not exceed $W$ if we take $N=cW^{1/\nu}$ with a suficiently small constant $c$. Then, the error bound \eqref{eq:f2infty} ensures the desired approximation rate \eqref{eq:mainineq} with $p=\frac{1}{\nu}$.

\section{Proof of Theorem \ref{th:main}}\label{sec:proofmain}
We divide the proof into three parts. In Section \ref{sec:twoscales} we construct the ``two-scales'' approximation $\widetilde f$ for the given function $f$ and estimate its accuracy. In Section \ref{sec:decode} we describe an efficient way to store and evaluate the refined approximation using the bit extraction technique. Finally, in Section \ref{sec:netimplem} we describe the neural network implementations of $\widetilde f$ and verify the network size constraints. 

\subsection{The two-scales approximation and its accuracy}\label{sec:twoscales}

We follow the outline of the proof given after the statement of Theorem \ref{th:main} and start by constructing the initial interpolating approximation $\widetilde f_1$ to $f$ using Proposition \ref{prop:shallow}. We may assume without loss of generality that $\widetilde f_1$ must be implemented by a network with not more than $W/2$ weights, reserving the remaining $W/2$ weights for the second approximation. Then $\widetilde f_1$ is given by Eq. \eqref{eq:wtf1}, where 
\begin{equation}\label{eq:ncwp}N=\lfloor c_1W^{1/\nu}\rfloor
\end{equation} with a sufficiently small constant $c_1$. The error of the  approximation $\widetilde f_1$ is given by Eq. \eqref{eq:f2infty}. We turn now to approximating the discrepancy $f_2=f-\widetilde f_1.$

\paragraph{Decomposition of the discrepancy.}
 It is convenient to represent $f_2$ as a finite sum of functions with supports consisting of disjoint ``patches'' of linear size $\sim\tfrac{1}{N}.$ Precisely, let $S=\{0,1,2\}^\nu$ and consider the partition of unity on the cube $[0,1]^\nu$
$$1=\sum_{\mathbf q\in S}g_{\mathbf q},$$
where
\begin{equation}\label{eq:gq}g_{\mathbf q}(\mathbf x)=\sum_{\mathbf n \in (\mathbf q+(3\mathbb Z)^\nu)\cap [0,N]^\nu}\phi(N\mathbf x-\mathbf n).\end{equation}
We then write 
\begin{equation}\label{eq:f2sum}f_2=\sum_{\mathbf q\in S}f_{2,\mathbf q},\end{equation}
where $$f_{2, \mathbf q}=f_2g_{\mathbf q}.$$
Each function $f_{2,\mathbf q}$ is supported on the disjoint union of  cubes $Q_\mathbf{n}=\bigtimes_{s=1}^\nu[\frac{n_s-1}{N},\frac{n_s+1}{N}]$ with $\mathbf n \in (\mathbf q+(3\mathbb Z)^\nu)\cap [0,N]^\nu$ corresponding to the spikes in the expansion \eqref{eq:gq}. With some abuse of terminology, we will refer to both these cubes and the respective restrictions $f_{2,\mathbf q}|_{Q_{\mathbf n}}$ as ``patches''. 

Since $\|g_{\mathbf q}\|_\infty\le 1,$ we have $\|f_{2, \mathbf q}\|_\infty\le \|f_{2}\|_\infty\le 3\nu \omega_f(\frac{1}{N})$ by Eq. \eqref{eq:f2infty}. Also, if $r\in[0,\frac{2}{N}]$ then 
\begin{align}\label{eq:lambda0}\omega_{f_{2,\mathbf q}}(r)&\le \|f_2\|_\infty \omega_{g_{\mathbf q}}(r)+\|g_{\mathbf q}\|_\infty \omega_{f}(r)\nonumber\\ 
&\le 3\nu \omega_f(\tfrac{1}{N})\sqrt{\nu}N r+\omega_f(r)\\
&\le (6\nu^{3/2}+1)\omega_f(r),\nonumber\end{align}  
where in the last step we used inequality \eqref{eq:alphaomega}.

\paragraph{The second (discrete) approximation.} We will construct the full approximation $\widetilde f$ in the form \begin{equation}\label{eq:ff1f2}\widetilde f=\widetilde f_1+\widetilde f_2=\widetilde f_1+\sum_{\mathbf q\in S}\widetilde f_{2,\mathbf q},\end{equation} where  $\widetilde f_{2,\mathbf q}$ are approximations to $f_{2,\mathbf q}$ on a smaller length scale $\frac{1}{M}$. We set 
\begin{equation}\label{eq:mcwp}M=c_2W^p,\end{equation} 
where $c_2$ is a sufficiently small constant to be chosen later and $p$ is the desired power in the approximation rate. We will assume without loss of generality that $\frac{M}{N}$ is integer, since in the sequel it will be convenient to consider the grid $(\mathbb Z/N)^\nu$  as a subgrid in the refined grid $(\mathbb Z/M)^\nu$.

We define $\widetilde f_{2,\mathbf q}$ to be piecewise linear  with respect to the refined triangulation $P_{M}$ and to be given on the refined grid $(\mathbb Z/M)^\nu$ by
\begin{equation}\label{eq:ft2qmn}\widetilde f_{2,\mathbf q}(\tfrac{\mathbf m}{M})=\lambda\lfloor f_{2,\mathbf q}(\tfrac{\mathbf m}{M})/\lambda\rfloor, \quad \mathbf m\in[0,\ldots, M]^\nu.\end{equation}
Here the discretization parameter $\lambda$ is given by
\begin{equation}\label{eq:lam0}\lambda =   (6\nu^{3/2}+1)\omega_f(\tfrac{1}{M})\end{equation}
so that, by Eq.~\eqref{eq:lambda0}, we have
\begin{equation}\label{eq:lambda}\omega_{f_{2,\mathbf q}}(\tfrac{1}{M})\le \lambda.\end{equation}

\paragraph{Accuracy of the full approximation.}
Let us estimate the accuracy $\|f-\widetilde f\|_\infty$ of the full approximation $\widetilde f$. First we estimate $\|f_{2,\mathbf q}-\widetilde f_{2,\mathbf q}\|_\infty$. Consider the piecewise linear function $\widehat f_{2,\mathbf q}$ defined similarly to $\widetilde f_{2,\mathbf q}$, but exactly interpolating $f_{2,\mathbf q}$ at the knots $(\mathbb Z/M)^\nu.$ Then, by Eq. \eqref{eq:ft2qmn}, $\|\widehat f_{2,\mathbf q}-\widetilde f_{2,\mathbf q}\|_\infty\le \lambda$, since in each simplex the difference $\widehat f_{2,\mathbf q}(\mathbf x)-\widetilde f_{2,\mathbf q}(\mathbf x)$ is a convex combination of the respective values at the vertices. Also, $\|\widehat f_{2,\mathbf q}-f_{2,\mathbf q}\|_\infty\le 3\nu\omega_{f_{2,\mathbf q}}(\frac{1}{M})$ by applying the interpolation error bound \eqref{eq:f2infty} with $M$ instead of $N$ and $f_{2,\mathbf q}$ instead of $f$. Using Eqs. \eqref{eq:lam0} and \eqref{eq:lambda}, it follows that 
\begin{align}\label{eq:ff}\|f_{2,\mathbf q}-\widetilde f_{2,\mathbf q}\|_\infty &\le \|f_{2,\mathbf q}-\widehat f_{2,\mathbf q}\|_\infty+\|\widehat f_{2,\mathbf q}-\widetilde f_{2,\mathbf q}\|_\infty\nonumber\\
&\le 3\nu\lambda+\lambda\\
&\le (3\nu+1)(6\nu^{3/2}+1)\omega_f(\tfrac{1}{M}).\nonumber\end{align}
Summing the errors over all $\mathbf q\in S$, we can bound the error of the full approximation using representations \eqref{eq:f2sum}, \eqref{eq:ff1f2} and the above bound:
\begin{align}\label{eq:finalerr}
\|f-\widetilde f\|_\infty &= \|f_2-\widetilde f_2\|_\infty\nonumber\\
&\le \sum_{\mathbf q\in S} \|f_{2,\mathbf q}-\widetilde f_{2,\mathbf q}\|_\infty \\ 
&\le 3^\nu(3\nu+1)(6\nu^{3/2}+1)\omega_f(\tfrac{1}{M}).\nonumber
\end{align}
Observe that, by Eq.\eqref{eq:mcwp}, this yields the desired approximation rate \eqref{eq:mainineq} provided we can indeed implement $\widetilde f$ with not more than $W$ weights:
\begin{equation}\label{eq:fwfbound}
\|f-\widetilde f\|_\infty\le 3^\nu(3\nu+1)(6\nu^{3/2}+1)\omega_f(c_2^{-1}W^{-p}).
\end{equation}  

\subsection{Storing and decoding the refined approximation}\label{sec:decode}

We have reduced our task to implementing  the functions $\widetilde f_{2,\mathbf q}$ subject to the network size and depth constraints. We describe now an efficient way to compute these functions using a version of the bit extraction technique.

\paragraph{Patch encoding.}

Fix $\mathbf q\in S$. Note that, like $f_{2,\mathbf q}$, the approximation $\widetilde f_{2,\mathbf q}$ vanishes outside of the cubes $Q_{\mathbf n}$ with $\mathbf n \in (\mathbf q+(3\mathbb Z)^\nu)\cap [0,N]^\nu$. Fix one of these $\mathbf n$ and consider those values $\lfloor f_{2,\mathbf q}(\tfrac{\mathbf m}{M})/\lambda\rfloor$ from Eq.\eqref{eq:ft2qmn} that lie in this patch: 
$$A_{\mathbf q, \mathbf n}(\mathbf m)=\lfloor f_{2,\mathbf q}(\tfrac{\mathbf n}{N}+\tfrac{\mathbf m}{M})/\lambda\rfloor, \quad \mathbf m\in[-\tfrac{M}{N},\ldots, \tfrac{M}{N}]^\nu.$$

By the bound \eqref{eq:lambda}, if $\mathbf m_1,\mathbf m_2$ are neighboring points on the grid $\mathbb Z^\nu$, then $A_{\mathbf q, \mathbf n}(\mathbf m_1)-A_{\mathbf q, \mathbf n}(\mathbf m_2)\in\{-1,0,1\}.$ Moreover, since $f_{2,\mathbf q}$ vanishes on the boundary of $Q_{\mathbf n}$, we have $A_{\mathbf q, \mathbf n}(\mathbf m)=0$ if one of the components of $\mathbf m$ equals $-\frac{M}{N}$ or $\frac{M}{N}$. Let us consider separately the first component in the multi-index $\mathbf m$ and write  $\mathbf m=(m_1,\overline{\mathbf m})$. 
Denote $$B_{\mathbf q, \mathbf n}({\mathbf m})=A_{\mathbf q, \mathbf n}(m_1,\overline{\mathbf m})-A_{\mathbf q, \mathbf n}(m_1+1,\overline{\mathbf m}), \quad \mathbf m\in [-\tfrac{M}{N}+1,\ldots, \tfrac{M}{N}-1]^\nu.$$ 
Since $B_{\mathbf q,\mathbf n}(\mathbf m)\in\{-1,0,1\},$ we can encode all the $(2M/N-1)^{\nu}$ values $B_{\mathbf q,\mathbf n}(\mathbf m)$ by a single ternary number $$b_{\mathbf q, \mathbf n}=\sum_{t=1}^{(2M/N-1)^{\nu}}3^{-t}(B_{\mathbf q,\mathbf n}(\mathbf m_t)+1),$$
where $t\mapsto\mathbf m_t$ is some enumeration of the multi-indices $\mathbf m$.  The values $b_{\mathbf q, \mathbf n}$ for all $\mathbf q$ and $\mathbf n$ will be stored as weights in the network and encode the approximation $\widetilde f_2$ in all the patches.
 
\paragraph{Reconstruction of the values $\{B_{\mathbf q,\mathbf n}(\mathbf m)\}$.}\label{sec:reconstrb}
The values $\{B_{\mathbf q,\mathbf n}(\mathbf m)\}$ are, up to the added constant 1, just the digits in the ternary representation of $b_{\mathbf q, \mathbf n}$ and can be recovered from $b_{\mathbf q, \mathbf n}$ by a deep ReLU network that iteratively implements ``ternary shifts''. Specifically, consider the sequence $z_t$ with $z_0=b_{\mathbf q, \mathbf n}$ and $z_{t+1}=3z_t -\lfloor 3z_t\rfloor$. Then $B_{\mathbf q,\mathbf n}(\mathbf m_t)=\lfloor 3z_{t-1} \rfloor-1$ for all $t$. To implement these computations by a ReLU network, we need to show how to compute $\lfloor 3z_{t}\rfloor$ for all $z_t$. Consider a piecewise-linear function $\chi_\epsilon:[0,3)\to \mathbb R$ such that \begin{equation}\label{eq:chi}\chi_\epsilon(x)=\begin{cases}0, & x\in[0,1-\epsilon],\\
1, & x\in[1,2-\epsilon],\\
2, & x\in[2,3-\epsilon].\end{cases}\end{equation}
Such a function can be implemented by $\chi_\epsilon (x)=\frac{1}{\epsilon}(x-(1-\epsilon))_+-\frac{1}{\epsilon}(x-1)_++\frac{1}{\epsilon}(x-(2-\epsilon))_+-\frac{1}{\epsilon}(x-2)_+.$ Observe that if $\epsilon< 3^{-(2M/N-1)^{\nu}}$, then for all $t$ the number $3z_t$ belongs to one of the three intervals in the r.h.s. of Eq.\eqref{eq:chi} and hence $\chi_\epsilon(3z_t)=\lfloor 3z_t\rfloor$. Thus, we can reconstruct the values $B_{\mathbf q,\mathbf n}(\mathbf m)$ for all $\mathbf m$ by a ReLU network with $O((M/N)^\nu)$ layers and weights. 

\paragraph{Computation of $\widetilde f_{2,\mathbf q}$ in a patch.} On the patch $Q_{\mathbf n}$, the function  $\widetilde f_{2,\mathbf q}$ can be expanded over the spike functions as
$$\widetilde f_{2,\mathbf q}(\mathbf x)=\lambda\sum_{\overline{\mathbf m}\in [-M/N+1,\ldots,M/N-1]^{\nu-1}}\sum_{m_1=-M/N+1}^{M/N-1}\phi(M(\mathbf x-(\tfrac{\mathbf n}{N}+\tfrac{(m_1, \overline{\mathbf m})}{M}))A_{\mathbf q, \mathbf n}(m_1, \overline{\mathbf m}).$$
It is convenient to rewrite this computation in terms of the numbers $B_{\mathbf q, \mathbf n}(\mathbf m)$ and the expressions 
\begin{equation}\label{eq:Phi}\Phi_{\mathbf n,\overline{\mathbf m}}(m_1, \mathbf x)=\sum_{s=-M/N+1}^{m_1}\phi(M(\mathbf x-(\tfrac{\mathbf n}{N}+\tfrac{(m_1, \overline{\mathbf m})}{M}))\end{equation}
using summation by parts in the direction $x_1$:
\begin{align}\widetilde f_{2,\mathbf q}(\mathbf x)
&=\lambda\sum_{\overline{\mathbf m}\in [-\frac{M}{N}+1,\ldots,\frac{M}{N}-1]^{\nu-1}}\sum_{m_1=-\frac{M}{N}+1}^{\frac{M}{N}}(\Phi_{\mathbf n,\overline{\mathbf m}}(m_1, \mathbf x)-\Phi_{\mathbf n,\overline{\mathbf m}}(m_1-1, \mathbf x))A_{\mathbf q, \mathbf n}(m_1, \overline{\mathbf m})\nonumber\\
&=\lambda\sum_{\overline{\mathbf m}\in [-\frac{M}{N}+1,\ldots,\frac{M}{N}-1]^{\nu-1}}\sum_{m_1=-\frac{M}{N}+1}^{\frac{M}{N}-1}\Phi_{\mathbf n,\overline{\mathbf m}}(m_1, \mathbf x)B_{\mathbf q, \mathbf n}(m_1, \overline{\mathbf m}),\label{eq:f2qnphib}
\end{align}
where we used the identities $A_{\mathbf q, \mathbf n}(-M/N, \overline{\mathbf m})=A_{\mathbf q, \mathbf n}(M/N, \overline{\mathbf m})=0$.

The above representation involves products $\Phi_{\mathbf n,\overline{\mathbf m}}(m_1, \mathbf x)B_{\mathbf q, \mathbf n}(m_1, \overline{\mathbf m}),$ and we show now how to implement them by a ReLU network. Note that $\Phi_{\mathbf n,\overline{\mathbf m}}(m_1, \mathbf x)\in[0,1]$ and $B_{\mathbf q, \mathbf n}(m_1, \overline{\mathbf m})\in\{-1,0,1\}$. But for any $x\in[0,1], y\in \{-1,0,1\}$ we can write
\begin{equation}\label{eq:xy}xy=(x+y-1)_++(-x-y)_+-(-y)_+.\end{equation}

\paragraph{Mapping $\mathbf x$ to the respective patch.} We will use the obtained formula \eqref{eq:f2qnphib} as a basis for computing $\widetilde f_{2,\mathbf q}(\mathbf x)$. But the formula is dependent on the patch the input $\mathbf x$ belongs to: the value $\Phi_{\mathbf n,\overline{\mathbf m}}(m_1, \mathbf x)$ depends on the patch index $\mathbf n$ by Eq.\eqref{eq:Phi}, and the values $B_{\mathbf q, \mathbf n}(m_1, \overline{\mathbf m})$ must be recovered from the patch-specific encoding weight $b_{\mathbf q,\mathbf n}$.  

Given $\mathbf x$, the relevant  value of $b_{\mathbf q, \mathbf n}$ can be selected among the values for all patches by the ReLU network computing
\begin{equation}\label{eq:bqx}b_{\mathbf q}(\mathbf x)=\sum_{\mathbf n \in (\mathbf q+(3\mathbb Z)^\nu)\cap [0,N]^\nu} \tfrac{b_{\mathbf q,\mathbf n}}{2} ((2-u)_+-(1-u)_+),\text{ where }u=\max_{s=1,\ldots,\nu}|Nx_s-n_s|.\end{equation}
If $\mathbf x$ belongs to some patch $Q_{\mathbf n},$ then $b_{\mathbf q}(\mathbf x)=b_{\mathbf q,\mathbf n}$, as required. If $\mathbf x$ does not belong to any patch, then $b_{\mathbf q}(\mathbf x)$ computes some garbage value which will be unimportant because we will ensure that all the factors $\Phi$ appearing in the sum \eqref{eq:f2qnphib} will vanish for such $\mathbf x$.

Now we show how to perform a properly adjusted computation of $\Phi_{\mathbf n,\overline{\mathbf m}}(m_1, \mathbf x)$. Namely, we will compute the function $\widetilde\Phi_{\overline{\mathbf m},\mathbf q}(m_1, \mathbf x)$ such that
\begin{numcases}{\widetilde\Phi_{\overline{\mathbf m},\mathbf q}(m_1, \mathbf x)=}
\Phi_{\mathbf n,\overline{\mathbf m}}(m_1, \mathbf x), & $\mathbf x\in Q_\mathbf{n}, \mathbf n \in (\mathbf q+(3\mathbb Z)^\nu)\cap [0,N]^\nu,$  \label{eq:phicase1}\\
0, &otherwise.\label{eq:phicase2}
\end{numcases}
First construct, using linear and ReLU operations, a map $\Psi_{\mathbf q}:[0,1]^\nu \to (\mathbb Z/N)^\nu$ mapping the patches $Q_\mathbf{n}$ to their centers: \begin{equation}\label{eq:psinn}\Psi_{\mathbf q}(\mathbf x)=\frac{\mathbf n}{N},\quad \text{if }\mathbf x\in Q_\mathbf{n}, \mathbf n \in (\mathbf q+(3\mathbb Z)^\nu)\cap [0,N]^\nu.\end{equation}
Such a map can be implemented by \begin{equation}\label{eq:psipsi}\Psi_{\mathbf q}(\mathbf x)=(\psi_{q_1}(x_1),\ldots,\psi_{q_\nu}(x_\nu)),\end{equation}
where \begin{equation}\label{eq:psiqn}\psi_{q}(x)=\frac{q}{N}+\sum_{k=0}^{\lceil N/3\rceil} ((x-\tfrac{q+3k+1}{N})_+-(x-\tfrac{q+3k+2}{N})_+). \end{equation}
Now if we try to define $\widetilde\Phi_{\overline{\mathbf m},\mathbf q}(m_1, \mathbf x)$ just by replacing $\frac{\mathbf n}{N}$ with $\Psi_{\mathbf q}(\mathbf x)$ in the definition \eqref{eq:Phi}, then we fulfill the requirement \eqref{eq:phicase1}, but not \eqref{eq:phicase2}. To also fulfill \eqref{eq:phicase2}, consider the auxiliary ``suppressing'' function \begin{equation}\label{eq:theta}\theta_\mathbf{q}(\mathbf x)=N\sum_{\mathbf n \in (\mathbf q+(3\mathbb Z)^\nu)\cap [0,N]^\nu}  (1-\max_{s=1,\ldots,\nu}|Nx_s-n_s|)_+.\end{equation}
If $\mathbf x$ does not lie in any patch $Q_{\mathbf n},$ then $\theta_\mathbf{q}(\mathbf x)=0.$ On the other hand, $\theta_\mathbf{q}(\mathbf x)\ge\Phi_{\mathbf n,\overline{\mathbf m}}(m_1, \mathbf x)$ for all $\mathbf n \in (\mathbf q+(3\mathbb Z)^\nu)\cap [0,N]^\nu, (m_1, \overline{\mathbf m})\in[-\frac{M}{N}+1,\frac{M}{N}-1]^\nu$ and $\mathbf x\in [0,1]^\nu$. It follows that if we set  
\begin{equation}\label{eq:tphimin}\widetilde\Phi_{\overline{\mathbf m},\mathbf q}(m_1, \mathbf x)=\min\Big(\sum_{s=-M/N+1}^{m_1}\phi(M(\mathbf x-(\Psi_{\mathbf q}(\mathbf x)+\tfrac{(m_1, \overline{\mathbf m})}{M}))), \theta_\mathbf{q}(\mathbf x)\Big),\end{equation}
then this $\widetilde\Phi_{\overline{\mathbf m},\mathbf q}(m_1, \mathbf x)$ satisfies both \eqref{eq:phicase1} and \eqref{eq:phicase2}. Then, based on \eqref{eq:f2qnphib}, we can write the final formula for  $\widetilde f_{2,\mathbf q}(\mathbf x)$ as
\begin{equation}\label{eq:ft2q}\widetilde f_{2,\mathbf q}(\mathbf x)=\lambda\sum_{\overline{\mathbf m}\in [-M/N+1,\ldots,M/N-1]^{\nu-1}}\sum_{m_1=-M/N+1}^{M/N-1}\widetilde\Phi_{\overline{\mathbf m},\mathbf q}(m_1, \mathbf x)B_{\mathbf q, \mathbf n}(m_1, \overline{\mathbf m}).
\end{equation}

\paragraph{Summary of the computation of $\widetilde f_{2,\mathbf q}$.}\label{sec:sumf2q} We summarize now how $\widetilde f_{2,\mathbf q}(\mathbf x)$ is computed by the ReLU network. 
\begin{enumerate}
\item The patch-encoding number $b_{\mathbf q}(\mathbf x)$ and the auxiliary functions $\Psi_{\mathbf q}(\mathbf x), \theta_\mathbf{q}(\mathbf x)$ are computed  by Eqs.\eqref{eq:bqx}, \eqref{eq:psinn}-\eqref{eq:psiqn}, \eqref{eq:theta}, respectively.
\item The numbers $B_{\mathbf q,\mathbf n}(\mathbf m)$ that describe the patch containing $\mathbf x$ are decoded from $b_{\mathbf q}(\mathbf x)$ by a deep network with $O((M/N)^\nu)$ layers and weights  as detailed in Section \ref{sec:reconstrb}. We choose the enumeration $t\mapsto\mathbf m_t$ so that same-$\overline{\mathbf m}$ multi-indices $\mathbf m$ form contiguous blocks of length $2M/N-1$ corresponding to summation over $m_1$ in \eqref{eq:ft2q}; the ordering of the multi-indices $\overline{\mathbf m}$ is arbitrary. 
\item The values $\widetilde\Phi_{\overline{\mathbf m},\mathbf q}(m_1, \mathbf x)$ are computed iteratively in parallel to the respective $B_{\mathbf q,\mathbf n}(\mathbf m)$, using Eq.\eqref{eq:tphimin}. For $m_1>-M/N+1$, the sum $\sum_{s=-M/N+1}^{m_1}\phi(\ldots)$ is computed by adding the next term to the sum $\sum_{s=-M/N+1}^{m_1-1}\phi(\ldots)$ retained from the previous step. This shows that all the values $\widetilde\Phi_{\overline{\mathbf m},\mathbf q}(m_1, \mathbf x)$ can be computed with $O((M/N)^\nu)$ linear and ReLU operations. 
\item The sum \eqref{eq:ft2q} is computed with the help of multiplication formula \eqref{eq:xy}.
\end{enumerate}

\subsection{Network implementations}\label{sec:netimplem}

\paragraph{The weight and depth bounds.} 
We can now establish statement a) of Theorem \ref{th:main}. Recall that we have already established the approximation rate in Eq. \eqref{eq:fwfbound}, but still need to show that the total number of weights is bounded by $W$. The initial approximation $\widetilde f_1$ was constructed so as to be implementable by a network architecture with $W/2$ weights, and we need to show now that the second approximation $\widetilde f_2$ can also be implemented by an architecture with not more than $W/2$ weights. 

By Eq. \eqref{eq:ff1f2}, computation of $\widetilde f_2$ amounts to $3^\nu$ computations of the terms $\widetilde f_{2,\mathbf q}$. These latter computations are detailed in stages 1--4 in Section \ref{sec:sumf2q}. Stage 1 can be performed by a parallellized ReLU network of depth $O(1)$ with $O(N^\nu)$ elementary operations. If the constant $c_1$ in our choice \eqref{eq:ncwp} of $N$ is small enough, then the total number of weights required for stage 1 does not exceed $W/4$.  Next, stages 2--4 require a deep network that has $O((M/N)^\nu)$ layers and weights. By our choice of $N,M$ in Eqs. \eqref{eq:ncwp}, \eqref{eq:mcwp} we can write $O((M/N)^\nu)$ as $O((c_2/c_1)^\nu W^{p\nu-1})$. Assuming $c_2$ is small enough, the number of weights for stages 2--4 does not exceed $W/4$ as well, and so the total number of weights required for $\widetilde f$ does not exceed $W$, as desired. In addition, the depth of the network implementation of $\widetilde f$ is determined by the depth of stages 2--4, i.e. is $O(W^{p\nu-1}),$ also as claimed in Theorem \ref{th:main}a).

We describe now in more detail the implementation of $\widetilde f$ by the narrow fully-connected or by the stacked architectures as stated in parts b) and c) of Theorem \ref{th:main}.

\paragraph{Implementation of $\widetilde f$ in the case $p=\frac{2}{\nu}$.}

In this case  we implement $\widetilde f$ by a narrow fully-connected network and estimate now its sufficient width. We reserve $\nu$ channels to pass forward the $\nu$ components of the input $\mathbf x$ and one channel to store partial results of the sum $\widetilde f(\mathbf x)=\widetilde f_1(\mathbf x)+\sum_{\mathbf q\in\{0,1,2\}^\nu}\widetilde f_{2,\mathbf q}(\mathbf x).$ The terms in this sum can be computed in a serial way, so we only need to estimate the network width sufficient for implementing $\widetilde f_1$ and any $\widetilde f_{2,\mathbf q}$. 

Given $\mathbf x$, the value $\phi(\mathbf x)$ or, more generally, $\phi(N\mathbf x-\mathbf n)$, can be computed by \eqref{eq:spike} using chained binary $\min$'s and just one additional channel. 
Therefore, by \eqref{eq:wtf1}, $\widetilde f_1$ can be implemented with the standard network of width $\nu+2$.

Now consider implementation of $\widetilde f_{2,\mathbf q}$ as detailed in Section \ref{sec:sumf2q}. In stage 1, we compute the numbers $b_{\mathbf q}(\mathbf x), \theta_\mathbf{q}(\mathbf x)$ and the $\nu$-dimensional vector $\Psi_{\mathbf q}(\mathbf x)$. 
This can be done with $\nu+3$ channels (one extra channel is reserved for storing partial sums). We reserve $\nu+1$ channels for the next stages to pass forward the values  $\Psi_{\mathbf q}(\mathbf x), \theta_\mathbf{q}(\mathbf x)$. In stage 2, we decode the numbers $B_{\mathbf q, \mathbf n}(\mathbf m)$ from $b_{q}(\mathbf x).$ This can be done with 4 channels: one is used to store and refresh the values $z_t$, another to output the  sequence $B_{\mathbf q, \mathbf n}(\mathbf m_t)$, and two more channels to compute $\chi_\epsilon(3 z_t)$. In stage 3,  we compute the numbers $\widetilde\Phi_{\overline{\mathbf m},\mathbf q}(m_1, \mathbf x)$. This can be done with 3 channels: one is used to keep partial sums 
$\sum_{s=-M/N+1}^{m_1}\phi(\ldots)$, another to compute current $\phi(\ldots)$, and the third to compute $\widetilde\Phi_{\overline{\mathbf m},\mathbf q}(m_1, \mathbf x)$ by Eq.\eqref{eq:tphimin}. Stages 2 and 3 are performed in parallel and their channels are aligned so that for each $\mathbf m$ the values $B_{\mathbf q, \mathbf n}(\mathbf m)$ and $\widetilde\Phi_{\overline{\mathbf m},\mathbf q}(m_1, \mathbf x)$ can be found in the same layer. In stage 4, we compute the sum \eqref{eq:ft2q}. This stage is performed in parallel to stages 2 and 3 and requires one more channel for computing the ReLU operations in the multiplication formula \eqref{eq:xy}. We conclude that $2\nu+10$ channels are sufficient for the whole computation of $\widetilde f_{2,\mathbf q}$ and hence for the whole computation of $\widetilde f.$ 

\paragraph{Implementation of $\widetilde f$ in the case $p\in(\frac{1}{\nu},\frac{2}{\nu})$.}
In this case the initial approximation $\widetilde f_1$ as well as stage 1 in the computations of $\widetilde f_{2,\mathbf q}$ are implemented by the parallel shallow subnetwork and only stages 2--4 of the computations of $\widetilde f_{2,\mathbf q}$ are implemented by the deep narrow subnetwork. We can make some blocks of the parallel subnetwork output the terms $f(\tfrac{\mathbf n}{N})\phi(N\mathbf x-\mathbf n)$ of the expansion of $\widetilde f_1$ while other blocks output the numbers $b_{\mathbf q}(\mathbf x), \theta_\mathbf{q}(\mathbf x)$ and the $\nu$-dimensional vector $\Psi_{\mathbf q}(\mathbf x)$. The stages 2--4 are then implemented as in the previous case $p=\frac{2}{\nu}$, but now all the $\mathbf q\in S$ must be processed in parallel -- that's why the width of the deep subnetwork is taken to be $3^\nu(2\nu+10).$ 

\section*{Acknowledgments}
The author thanks Alexander Kuleshov and the anonymous referees for helpful comments and suggestions. The research was supported by the Skoltech SBI Bazykin/Yarotsky project.

\bibliography{relu}

\end{document}